%% file: main.tex
\documentclass[letterpaper]{article} 
\usepackage[]{aaai2026}  
\usepackage{times}  
\usepackage{helvet}  
\usepackage{courier}  
\usepackage[hyphens]{url}  
\usepackage{graphicx} 
\usepackage{natbib}  
\usepackage{caption} 
\usepackage{algorithm}
\usepackage{algorithmic}

\usepackage{amsmath, amssymb, mathabx}
\usepackage{xcolor}
\usepackage{tikz}
\usetikzlibrary{calc}
\usepackage{pgfplots}
\pgfplotsset{compat=1.18}
\usepackage{subcaption}
\usepackage[export]{adjustbox}
\usepackage{booktabs}

\usepackage{newfloat}
\usepackage{listings}
\DeclareCaptionStyle{ruled}{labelfont=normalfont,labelsep=colon,strut=off} 
\floatstyle{ruled}
\newfloat{listing}{tb}{lst}{}
\floatname{listing}{Listing}
\title{Improving Stochastic Action-Constrained \\ Reinforcement Learning via Truncated Distributions}
\author{
Roland Stolz,
Michael Eichelbeck,
Matthias Althoff
}
\affiliations{
Technical University of Munich \\
{
\small\texttt{\{roland.stolz, michael.eichelbeck, althoff\}@tum.de}
}
}

\usepackage{amsthm}
\newtheorem{problem}{Problem}
\newtheorem{proposition}{Proposition}

\newcommand{\As}{\mathcal{A}^s}
\newcommand{\pol}[1]{\pi_\theta({#1} \vert s)}
\newcommand{\poldot}{\pol{\cdot}}
\newcommand{\pola}{\pol{a}}
\newcommand{\pols}[1]{\pi_\theta^s({#1} \vert s)}
\newcommand{\polsdot}{\pols{\cdot}}
\newcommand{\polsa}{\pols{a}}
\newcommand{\mvtd}[1]{f(x; {#1})}
\DeclareMathOperator*{\argmax}{arg\,max}
\DeclareMathOperator*{\argmin}{arg\,min}

\usepackage{acronym}
\newacro{RL}{reinforcement learning}
\newacro{PPO}{proximal policy optimization}
\newacro{SAC}{soft actor-critic}
\newacro{PDF}{probability density function}
\newacro{CDF}{cumulative distribution function}
\newacro{RDHR}{random direction hit-and-run}

\begin{document}

\maketitle

\begin{abstract}

    In reinforcement learning (RL), it is often advantageous to consider additional constraints on the action space to ensure safety or action relevance. Existing work on such action-constrained RL faces challenges regarding effective policy updates, computational efficiency, and predictable runtime.
	Recent work proposes to use truncated normal distributions for stochastic policy gradient methods. However, the computation of key characteristics, such as the entropy, log-probability, and their gradients, becomes intractable under complex constraints. Hence, prior work approximates these using the non-truncated distributions, which severely degrades performance.
	We argue that accurate estimation of these characteristics is crucial in the action-constrained RL setting, and propose efficient numerical approximations for them.
	We also provide an efficient sampling strategy for truncated policy distributions and validate our approach on three benchmark environments, which demonstrate significant performance improvements when using accurate estimations.
    
\end{abstract}


\section{Introduction}
\input{sections/introduction.tex}

\section{Preliminaries}\label{sec:prelim}
\input{sections/preliminaries}

\section{Problem Statement}\label{sec:problem_statement}
\input{sections/problem_statement}

\section{Using Truncated Distributions in RL}\label{sec:method}
\input{sections/method}

\section{Experiments}\label{sec:experiments}
\input{sections/experiments}

\section{Conclusion}\label{sec:conclusion}
\input{sections/conclusion}

\section*{Acknowledgements and Disclosure of Funding}
We thank Paul Moritz Koebe for contributing to the implementation of the Seeker environment. 
We gratefully acknowledge that this project was funded by the Deutsche Forschungsgemeinschaft (DFG, German Research Foundation) – SFB 1608 – 501798263; and DFG 458030766.

\small
\bibliography{aaai2026}

\pagebreak
\normalsize
\appendix
\input{sections/appendix}

\end{document}

%% file: sections/introduction.tex
\Ac{RL} natively operates on a static action space from which the agent can choose an action in each time step. However, a dynamic restriction of the action space can be advantageous to exclude irrelevant actions \cite{stolz_excluding_2024} or even necessary to guarantee safety \cite{krasowski2023provably}. The field of action-constrained \ac{RL} has developed various approaches to achieve this with zero constraint violations. 
Most methods handle constraints by projecting actions onto sets formed by the action constraints \cite{kasaura_benchmarking_2023}, which can lead to zero gradients, because multiple actions outside the corresponding set might be mapped to the same action \cite{lin_escaping_2021, kasaura_benchmarking_2023}. Some approaches address this issue by attempting to retain the gradient through learning a flow model \cite{brahmanage_flowpg_2023, brahmanage_leveraging_2025}, or using Franke-Wolfe optimization \cite{lin_escaping_2021}.

Other action-constrained \ac{RL} methods use different mappings from actions outside the constraints to within the set, such as $\alpha$-projection \cite{sanket_solving_2020}, replacement with fail-safe actions \cite{krasowski2023provably}, or radially contracting the action space \cite{kasaura_benchmarking_2023, stolz_excluding_2024}. 
One proposal by \cite{stolz_excluding_2024} is to learn actions in a box-constrained latent space that is mapped to the constrained action space via a linear transformation, which is efficient, but restricted to constraints that can be formulated as zonotopes. The work by \cite{theile_learning_2024} uses a similar approach, but learns the mapping from the latent space to the action constraints.
Another recent work relies on rejection sampling to obtain feasible actions \cite{hung_efficient_2025}. While a multi-objective algorithm is developed in that work to incentivize the agent to follow trajectories with large feasible action sets, rejection sampling cannot provide guarantees with regard to computation time, which we illustrate as part of our discussion in Sec. \ref{sec:sampling}.

Instead of defining a mapping to constraint-satisfying actions, the study in \cite{stolz_excluding_2024} proposes to directly truncate the policy distribution using the action constraints. However, this has two limitations. First, the policy update is approximated using the non-truncated distribution, because the metrics required by the \ac{RL} objectives, such as entropy and log-probability, are generally intractable under the constrained distribution.
Second, the sampling approach is based on a geometric random walk, which has high computational costs per sample and prevents using the reparameterization trick \cite{kingma_auto-encoding_2014} for gradient estimation in stochastic \ac{RL}.

This paper builds on the work employing truncated distributions for \ac{RL} \cite{stolz_excluding_2024} and tackles key limitations by developing more expressive policy updates and efficient sampling methods. In particular, our contributions are:
\begin{itemize}
    \item Approximations of the intractable log-probability and entropy for truncated distributions applicable to convex, non-convex, and disjoint sets;
    \item An efficient, hybrid sampling algorithm for truncated distributions, merging rejection sampling and geometric random walks for improved performance;
    \item A novel application of the reparameterization trick to geometric random walks, enabling differentiable sampling.
\end{itemize}

The remainder of this study is organized as follows. After introducing preliminaries (Sec. \ref{sec:prelim}), we formalize our problem statement (Sec. \ref{sec:problem_statement}), followed by our proposed methods for leveraging truncated distributions in \ac{RL} (Sec \ref{sec:method}). 
Finally, we compare the developed mechanisms on three \ac{RL} benchmarks (Sec. \ref{sec:experiments}) and provide concluding remarks (Sec. \ref{sec:conclusion}).

%% file: sections/preliminaries.tex
\subsection{Action-Constrained Markov Decision Processes}

We consider problems that can be modeled by a Markov decision processes, defined as a tuple $\left( \mathcal{S}, \mathcal{A}, T, r, \gamma \right)$ comprising the following components: an observable and continuous state space $\mathcal{S} \subseteq \mathbb{R}^{m}$, an action space $\mathcal{A} \subseteq \mathbb{R}^{n}$, a stationary state-transition distribution $T(s'| a, s)$ that characterizes the probability of transitioning to a subsequent state $s' \in \mathcal{S}$ given the current state $s \in \mathcal{S}$ and executed action $a \in \mathcal{A}$, a reward function $r: \mathcal{S} \times \mathcal{A} \rightarrow \mathbb{R}$, and a discount factor $\gamma$ for future rewards \cite{sutton_reinforcement_2018}. 
Action-constrained \ac{RL} further considers a state-dependent feasible action space $\As(s) \subseteq \mathcal{A}$. 
The goal is to learn a policy $\polsa$ parameterized by $\theta$
that maximizes the expected discounted return $\max_\theta \mathbb{E}_{\pi_\theta} \sum_{t=0}^{\infty} \gamma^t r(s_t, a_t)$, while only selecting actions from the feasible action space $\mathcal{A}^s(s_t)$ at each time step $t$.

\subsection{Stochastic Policy RL}

Stochastic, on-policy algorithms, such as \ac{PPO}, learn the optimal policy $\pi_\theta^*(a \vert s)$ by updating the parameters with $\theta \leftarrow \theta + \beta \nabla_{\theta} J(\pi_\theta)$, according to the policy gradient \cite[Thm. 2]{sutton_policy_1999}
\begin{equation}\label{eq:policy_gradient}
    \nabla_{\theta} J(\pi_\theta) = \mathbb{E}_{\pi_\theta} \left[ \nabla_\theta \log \pi_\theta(a | s) 
    A_{\pi_\theta} (a, s) \right],
\end{equation}
where $A_{\pi_\theta}(a, s)$ denotes the advantage function, which quantifies the expected improvement in return from executing action $a$ in state $s$ relative to the expected performance under the current policy $\pi_\theta(a \vert s)$, and is typically approximated using a neural network.

As an alternative, the stochastic, off-policy algorithm \ac{SAC} \cite{haarnoja_soft_2018} incorporates entropy regularization into the optimization objective. \ac{SAC} aims to maximize the expected cumulative return augmented with an entropy term
\begin{equation}\label{eq:sac_objective}
    J(\pi_\theta) = \mathbb{E}_{\pi_\theta} \left[ \sum_{t=0}^{T} \gamma^t \Big( r(s_t, a_t) + \alpha \mathcal{H} \big( \pi_\theta(\cdot \vert s_t)\big) \Big) \right],
\end{equation}
where $\mathcal{H}(\pi_\theta(\cdot \vert s_t)) = -\mathbb{E}_{a \sim \pi_\theta} [\log \pi_\theta(a_t | s_t)]$ represents the policy entropy and $\alpha$ is a temperature parameter controlling the trade-off between exploration and exploitation. 
The optimal policy $\pi_\theta^*(a \vert s)$ is obtained via gradient descent on the parameters $\theta$ with the gradient
\begin{equation}\label{eq:sac_gradient}
\begin{aligned}
    \nabla_\theta& J(\pi_\theta) = \mathbb{E}_{s \sim \mathcal{D}} \big[\nabla_\theta \log \pi_\theta (a \vert s) + \\
    & \big( \alpha \nabla_a \log\pi_\theta(a \vert s) - \nabla_a Q_{\phi}(s, a) \big) \nabla_\theta a \big\vert_{a \sim \pi_\theta} \big],
\end{aligned}
\end{equation}
where $\mathcal{D}$ denotes the replay buffer containing experience tuples, and $Q_\phi$ is the soft Q-function approximated by critics with parameters $\phi$. Evaluating this gradient requires backpropagating through $a \sim \pi_\theta$, which is achieved using the reparameterization trick \cite{kingma_auto-encoding_2014}.

\subsection{Set Representations}

A multidimensional interval $\mathcal{I} \subset \mathbb{R}^d$ is defined by lower and upper bounds $l, u \in \mathbb{R}^d$, such that
\begin{equation}
    \mathcal{I} = [l, u] = \{x \in \mathbb{R}^d : l \leq x \leq u\}.
\end{equation}
A polytope $\mathcal{P} \subset \mathbb{R}^d$ can be represented as the intersection of $m$ halfspaces and is denoted by
\begin{equation}
    \mathcal{P}  = \{x \in \mathbb{R}^d : Ax \leq b\},
\end{equation}
where $A \in \mathbb{R}^{m \times d}$ and $b \in \mathbb{R}^m$.

\subsection{Truncated Distributions}
We write a continuous \ac{PDF} $f(x) \in \mathbb{R}^d$ truncated to a set $\mathcal{W}$ as 
\begin{equation}\label{eq:truncated_general}
   \mvtd{\mathcal{W}} = \frac{f(x) \cdot \phi(x; \mathcal{W})}{Z_\mathcal{W}},
\end{equation}
where $\phi(x; \mathcal{W})$ is the indicator function that returns 1 if $x \in \mathcal{W}$ and 0 otherwise, and $Z_\mathcal{W} = \int_{x \in \mathcal{W}} f(x) \mathrm{d} x$ is the normalizing constant. The function is generally intractable due to the integral in the denominator. However, for univariate $f(x)$, and when  $\mathcal{W} = [ l, u ] \subset \mathbb{R}$, it can be written as \cite[Eq.~13.133]{johnson_continuous_1994}
\begin{equation}\label{eq:truncated_univariate}
	\mvtd{[l, u]} = \begin{cases}
		\frac{f(x)}{Z_{[l,u]}} & \text{if } l \leq x \leq u, \\
		0,             & \text{otherwise},
	\end{cases}
\end{equation}
where $Z_{[l,u]} = F(u) - F(l)$, and $F(x)$ is the \ac{CDF} of $f(x)$ (see Fig.~\ref{fig:truncated_normal}).
The entropy in general is \cite[Eq.~2.1]{shangari_partial_2012}
\begin{equation}\label{eq:entropy_truncated_univariate}
    \begin{aligned}
	\mathcal{H}\big( &f(\cdot; [l, u]) \big) = \\
    &-\frac{1}{Z_{[l,u]}} \int_l^u f(x) \log f(x) \mathrm{d}x + \log Z_{[l,u]},
    \end{aligned}
\end{equation}
which, for Gaussian distributions $\mathcal{N}(\mu, \sigma^2)$, has the closed from \cite[Sec.~4.26]{michalowicz_handbook_2013}
\begin{equation}\label{eq:entropy_truncated_gaussian}
    \begin{aligned}
        \mathcal{H}\big( &f(\cdot; [l, u]) \big) = \\
        &\frac{1}{2} \log(2\pi e \sigma^2) + \log(Z_{[l,u]})
         - \frac{u' \, \varphi(u') - l' \, \varphi(l')}{2Z_{[l,u]}},
    \end{aligned}
\end{equation}
where $u' = \frac{u - \mu}{\sigma}$, $l' = \frac{l - \mu}{\sigma}$, and $\varphi(x)$ is the \ac{PDF} of the standard normal distribution. 
The mode of a truncated distribution (i.e., the point at which its \ac{PDF} attains the maximum value) for a Gaussian $f(x)$ is
\begin{equation}\label{eq:mode_truncated_dist}
	\argmax_{x} f(x; [l,u]) =
	\begin{cases}
		l,   & \text{if } \mu \leq l, \\
		u,   & \text{if } \mu \geq u, \\
		\mu, & \text{otherwise.}
	\end{cases}
\end{equation}
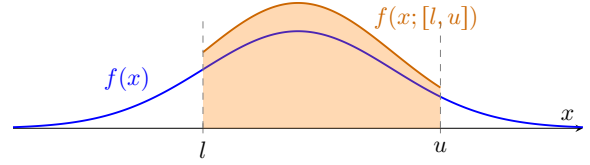
\begin{figure}[!t]
    \centering
    \input{figures/truncated_normal.tex}
    \caption{The \ac{PDF} of the standard normal distribution $f(x)$ is truncated to the interval $[l, u]$ to obtain the truncated \ac{PDF} $f(x; [l, u])$. The area under the curve is normalized to one.}
    \label{fig:truncated_normal}
\end{figure}

\subsection{Sampling}\label{sec:prelim_sampling}
We can obtain a sample $q$ from a univariate truncated distribution $f(\cdot; [l, u])$ by first sampling a uniform random variable $y \sim \mathcal{U}(0, 1)$ and then applying the inverse \ac{CDF}\footnote{From "The Truncated Normal Distribution", Sec.~3.4, \url{https://people.sc.fsu.edu/~jburkardt/presentations/truncated_normal.pdf}}:
\begin{equation} \label{eq:inverse_transform_sampling}
	q = F^{-1} \Big(F(l) + y \big( F(u) - F(l) \big) \Big).
\end{equation}
For a multi-variate truncated distribution, we can obtain samples analytically only in special cases, which we explore in Sec.~\ref{sec:analytic_solution}. In the general case, we must use numerical alternatives, the most straightforward of which is rejection sampling. There, samples are drawn as $q \sim f(\cdot)$, and accepted as samples from $f(\cdot; \mathcal{W})$, if $q \in \mathcal{W}$.
As an alternative, \cite{stolz_excluding_2024} propose to use Markov chain Monte Carlo sampling to draw from $f(\cdot; \mathcal{W})$, specifically, the \ac{RDHR} algorithm \cite{gass_hit-and-run_2013,chalkis_volesti_2020}. 



%% file: figures/truncated_normal.tex
\begin{tikzpicture}[scale=0.9]
	\def\Z{0.7745}
	\begin{axis}[
			domain=-3:3,
			samples=200,
			axis x line=middle,
			axis y line=none,
			width=10cm,      
			height=3.2cm,     
			xlabel={$x$},
			xtick={-1, 1.5},
			xticklabels={$l$, $u$},
			ytick=\empty,
			ymin=0,
			ymax=0.45,
			clip=false,
		]
		\addplot [thick, blue, domain=-3:3] {1/(sqrt(2*pi))*exp(-0.5*x^2)};

		\addplot [
			domain=-1:1.5,
			samples=100,
			fill=orange,
			opacity=0.3,
			draw=none
		] {1/(sqrt(2*pi))*exp(-0.5*x^2)/\Z} \closedcycle;

		\addplot [
			thick,
			orange!80!black,
			domain=-1:1.5,
			samples=100
		] {1/(sqrt(2*pi))*exp(-0.5*x^2)/\Z};

		\addplot [dashed, gray] coordinates {(-1,0) (-1,0.45)};
		\addplot [dashed, gray] coordinates {(1.5,0) (1.5,0.45)};

		\node[orange!80!black] at (1.35,0.37) [above] {$f(x;[l,u])$};
		\node[blue] at (-1.8,0.12) [above] {$f(x)$};

	\end{axis}
\end{tikzpicture}

%% file: sections/problem_statement.tex
We consider the following setting for action-constrained RL in this work.
The set of action constraints $\As(s)$ is given for each state $s \in \mathcal{S}$. We directly truncate the original policy distribution $\pol{\cdot}$ to obtain the \ac{PDF} of the satisfying policy distribution $\pols{\cdot}$ as \cite{stolz_excluding_2024}
\begin{equation}\label{eq:truncated_policy_dist}
	\polsa = \frac{\pola \phi(a; \As)}{\int_{\tilde{a} \in \As(s)} \pol{\tilde{a}} \mathrm{d}\tilde{a}}
	= \frac{\pola \phi(a; \As)}{Z_{\As}(\theta)},
\end{equation}
where $\phi(a;\As)$ is the indicator function that returns 1 if $a \in \As(s)$ and 0 otherwise. Since we want to use $\polsdot$ for stochastic policy gradient algorithms, we need to solve the following problems.

\begin{problem}
Compute a differentiable expression for the log-probability of an action $\log \polsa$ with the gradient $\nabla_\theta \log \polsa = \nabla_\theta \log \pola - \nabla_\theta \int_{\tilde{a} \in \As(s)} \pol{\tilde{a}} \mathrm{d}\tilde{a} $.
\end{problem}

\begin{problem}
Compute the entropy $\mathcal{H}\big(\polsdot\big)$.
\end{problem}

\begin{problem}
Compute the mode of the truncated policy distribution $\argmax_{a \in \As} \polsa$ to evaluate the \ac{RL} agent.
\end{problem}

\begin{problem}
Sample efficiently from $\polsdot$.
\end{problem}

The difficulty of each problem depends on the nature of the set $\As$ and $\pol{\cdot}$. There are analytical solutions for a special case, which we detail in Sec.~\ref{sec:analytic_solution}. For more general cases, we propose approximations in Sec.~\ref{sec:approximations}, and Sec.~\ref{sec:sampling} proposes strategies for sampling from $\polsdot$.

%% file: sections/method.tex
\newcommand{\polsmarginal}{{\pi_\theta^s}^{(i)}(a_i \vert s)}
\newcommand{\polsdotmarginal}{{\pi_\theta^s}^{(i)}(\cdot \vert s)}
\newcommand{\Ast}{\tilde{\mathcal{A}}^s}

\subsection{Analytical Solution}\label{sec:analytic_solution}
When the policy distribution can be fully factorized into each action dimension $i$, i.e., $\pola = \prod_{i=1}^n \pi_\theta^{(i)}(a_i \vert s)$, and $\As$ is an interval $[l, u]$, the truncated policy distribution can also be factorized as
\begin{equation}\label{eq:truncated_policy_dist_factorized}
	\polsa = \prod_{i=1}^n \polsmarginal = \prod_{i=1}^n f(a_i; [l_i, u_i]),
\end{equation}
where the $i$-th marginal distribution $\polsmarginal$ is a univariate, truncated distribution as defined in \eqref{eq:truncated_univariate}. Using the independent dimensions, we can define the entropy from the sum of the marginals \cite[Cor.~8.6.2]{cover_differential_2005} as
\begin{equation}\label{eq:entropy_factorized}
	\mathcal{H}\big(\polsdot\big) = \sum_{i=1}^n \mathcal{H}\big(\polsdotmarginal\big).
\end{equation}
When the original policy distribution $\poldot$ is Gaussian, which is the case in most stochastic policy gradient \ac{RL} algorithms, such as \ac{PPO} \cite{schulman_proximal_2017} or \ac{SAC} \cite{haarnoja_soft_2018}, we can use \eqref{eq:entropy_truncated_gaussian} for the marginal entropies in \eqref{eq:entropy_factorized} to receive a closed-form expression for the entropy of the truncated policy distribution.
Finally, samples from the truncated distribution $\polsdot$ can be obtained by using inverse transform sampling \eqref{eq:inverse_transform_sampling} with each marginal distribution $\polsdotmarginal$ independently.

\subsection{Approximations}\label{sec:approximations}
To accurately compute the log-probability $\log \polsa$ in \eqref{eq:truncated_policy_dist}, we need to evaluate the normalizing constant $Z_{\As}(\theta) = \int_{\tilde{a} \in \As(s)} \pol{\tilde{a}} \mathrm{d}\tilde{a}$.
The integral is generally intractable, but it can be approximated well by using numerical integration methods, such as Monte Carlo integration \cite{robert_monte_2004} or cubature methods \cite{genz_adaptive_2003}. However, the first requires many samples to achieve a low error, and for the second, the number of evaluations of the functions scales exponentially with the dimension \cite{genz_adaptive_2003}.
Also, both methods make it difficult to back-propagate through the estimate of $Z_{\As}(\theta)$, hence the gradient would have to be approximated as $\nabla_\theta \log \polsa \approx \nabla_\theta \pola$. Therefore, we propose different methods for approximating the normalizing constant, which vary depending on the nature of the set $\As$. We first present methods for simple convex sets, followed by non-convex sets.

\subsubsection{Convex Sets}
In order to utilize the analytical solutions in Sec.~\ref{sec:analytic_solution} for general convex sets, we propose to approximate the convex set $\As$ as an interval $\mathcal{I}$, and use the metrics of the policy distribution truncated to $\mathcal{I}$ (which we denote as $\polsdot_\mathcal{I}$) as approximations for $\polsdot$:
\begin{align}
	\log \polsa           & \approx \log \pola - \log Z_{\mathcal{I}}(\theta), \label{eq:log_prob_approximation} \\
	\mathcal{H}(\polsdot) & \approx \mathcal{H}(\poldot_\mathcal{I}) \label{eq:entropy_approximation}.
\end{align}
We obtain the inner interval approximation $\mathcal{I}_\text{inner}$ of a convex set $\mathcal{S}$ by maximizing the geometric mean of its diameter:
\begin{equation}\label{eq:inner_interval_approximation}
	\begin{aligned}
		\max \quad              & \big( \prod_{i=1}^d (u_i - l_i)\big)^\frac{1}{d}    \\
		\text{subject to} \quad & \mathcal{I}_\text{inner}=[l,u] \subset \mathcal{S}. \\
	\end{aligned}
\end{equation}
The outer approximation is the tightest interval enclosed of $\As$, which can generally be obtained through the support functions in direction of the standard basis vectors \cite[Eq.~1]{althoff_support_2016}, or specifically for zonotopes and polytopes as in \cite[Prop.~2.2, Prop.~2.3]{althoff_reachability_2010}.

Since the intervals are nested such that $\mathcal{I}_\text{inner} \subseteq \As \subseteq \mathcal{I}_\text{outer}$, we can establish the general bounds $\square_{\mathcal{I}_\text{inner}} \leq \square \leq \square_{\mathcal{I}_\text{outer}}$ for the normalizing constant $\square = Z_{\As}(\theta)$ and for the entropy $\square = \mathcal{H}(\polsdot)$.
This holds for the former, because the integral over a nonnegative function is monotonically increasing when enlarging the integration domain. For the latter, the entropy of a truncated distribution $f(x;[l,u])$ is shown to be an increasing function if the \ac{CDF} $F(x)$ is log-concave, which is the case for Gaussian distributions \cite[Thm.~2.3]{shangari_partial_2012}.
These bounds enable us to interpolate the metric $\square$ between the inner and outer interval to achieve a tighter approximation than using either bound alone. We interpolate with
\begin{equation}
	\square_{\mathcal{I}_\text{combined}} = \left(1 - \frac{1}{2^d}\right) \square_{\mathcal{I}_\text{inner}} + \frac{1}{2^d} \square_{\mathcal{I}_\text{outer}},
\end{equation}
because the volume of $\As$ compared to the outer interval approximation decreases exponentially with the dimension~$d$ \cite[Thm~13.2.1]{matousek_volumes_2002}.

Finally, when $\As$ is convex and $\polsdot$ is a factorized Gaussian distribution, the mode is obtained by the action in $\As$, which minimizes the Mahalanobis distance to the mean:
\begin{equation}\label{eq:mode_convex}
	\begin{aligned}
		\argmin_a \quad         & (a - \mu)^\top \operatorname{diag}(\sigma^{-2}) (a - \mu) \\
		\text{subject to} \quad & a \in \As.
	\end{aligned}
\end{equation}

\subsubsection{Non-Convex and disjoint sets}\label{sec:non_convex_sets}
\newcommand{\ui}{\mathcal{U}_\mathcal{I}}
\newcommand{\finti}{f(x; \mathcal{I}^{(i)})}
\newcommand{\ii}{\mathcal{I}^{(i)}}
Under the very weak assumption that $\As$ is a measurable set, we can approximate it to arbitrary accuracy with a finite union of intervals, though the complexity increases exponentially with the dimension \cite[Ch. 3]{munkres_analysis_1991}.
Using this, we can approximate $\polsdot$ for any bounded $\As$ by truncating the original policy distribution to the union of $k$ intervals $\ui = \bigcup_{i=1}^k \mathcal{I}^{(i)}$.
\begin{proposition}\label{prop:truncated_policy_dist_union}
	Given the distribution truncated to the i-th interval $\finti$, the \ac{PDF} of the distribution $f(x)$ truncated to the union of $k$ non-overlapping intervals $\ui$ is
	\begin{equation}\label{eq:truncated_policy_dist_union}
		f(x; \ui) = \frac{f(x) \cdot \phi(x; \ui)}{Z_\mathcal{U}} = \sum_{i=1}^k w_i \finti,
	\end{equation}
	where $Z_\mathcal{U}$ is the normalizing constant of $f(x; \ui)$, and $w_i = \frac{Z_{\mathcal{I}^{(i)}}}{Z_\mathcal{U}}$ is the relative probability mass of the i-th interval.
\end{proposition}
\begin{proof}
	The result is proven in Appendix~\ref{app:truncated_union}.
\end{proof}
\begin{proposition}
	The entropy of a \ac{PDF} $f(x; \ui)$ truncated to the union of $k$ non-overlapping intervals $\ui$ is
	\begin{equation}\label{eq:entropy_union}
		\mathcal{H}(f(x;\ui)) = -\sum_{i=1}^k w_i \log w_i + \sum_{i=1}^k \mathcal{H}(\finti).
	\end{equation}
\end{proposition}
\begin{proof}
	The result is proven in Appendix~\ref{app:entropy_union}.
\end{proof}
The accuracy of the approximation $\polsa \approx \pi_\theta(a \vert s)_{\ui}$ and $\mathcal{H}(\polsdot) \approx \mathcal{H}(\pi^s_\theta(\cdot \vert s)_{\ui})$ depends on the method used for approximating $\As$ with $\ui$.
A good approximation can, for instance, be obtained with the approach from \cite{brimkov_object_2000}.

The mode $\argmax_{a \in \As} \polsa$ can be obtained using \eqref{eq:mode_convex}, although the non-convex set $\As$ makes it a non-convex optimization problem, which is generally NP-hard \cite{murty_np-complete_1987}.
If $\As$ can be inner-approximated by a convex set, an under-approximation of the mode is obtained by solving \eqref{eq:mode_convex} for the inner-approximating set.

\subsection{Sampling}\label{sec:sampling}
\subsubsection{Efficient sample generation}
Rejection sampling as described in Sec. \ref{sec:prelim_sampling} is often efficient, but suffers from the major drawback of a potentially very low acceptance rate when $Z_{\As}(\theta)$ is small; in extreme cases, the sampling process can even become stuck.
At the same time, \Ac{RDHR} exhibits high computational costs per sample, but guaranteed convergence in polynomial time $\mathcal{O}(n^3)$ for well-conditioned sets (i.e., the ratio of the radii of its minimum enclosing ball and maximum contained ball is bounded by $\mathcal{O}(\sqrt{n})$) \cite{lovasz_hit-and-run_2006}.

To combine the strengths of both approaches, we propose to first execute rejection sampling for a maximum of $M$ attempts, and switch to the \ac{RDHR} algorithm when no sample has been accepted. This leverages the potentially fast rejection sampling, while ensuring sample generation via \ac{RDHR} in $\mathcal{O}(n^3)$ even when $Z_{\As}(\theta)$ is very small.

\subsubsection{Differentiating through sampling}
In \ac{SAC}, the policy gradient in \eqref{eq:sac_gradient} contains $\nabla_a Q_\phi(s,a)\nabla_\theta a \vert_{a \sim \pi_\theta}$, which requires differentiating through $a \sim \polsdot$ \cite{haarnoja_soft_2018}.
To achieve this, the reparameterization trick \cite{kingma_auto-encoding_2014} is employed, which expresses the sample as a differentiable function of an independent random variable $\varepsilon \in \mathbb{R}^d$ as $a = f_\theta(\varepsilon, s)$. This allows us to write the gradient as $\nabla_a Q_\phi\big(s,f_\theta(\varepsilon, s)\big) \nabla_\theta f_\theta(\varepsilon, s)$.
More specifically, for Gaussian distributions $\mathcal{N}(\mu, \Sigma)$, where $\theta = \{\mu, L\}$ and $\Sigma = L L^T$, $\varepsilon \sim \mathcal{N}(0, I)$ is sampled from the standard Gaussian distribution and transformed as $a = \mu + L \, \varepsilon$.

Accordingly, rejection sampling can simply be made differentiable by proposing the samples with $a = f_\theta(\varepsilon, s)$. Also, the inverse transform sampling proposed in Sec.~\ref{sec:analytic_solution} is directly differentiable \cite{kingma_auto-encoding_2014}. However, this is not the case for \ac{RDHR} sampling, hence we derive the differentiable reparameterization for \ac{RDHR}.
\begin{proposition}
	Let $\polsdot$ be a multivariate Gaussian distribution $\mathcal{N}_{\As}(\mu, \Sigma)$ truncated to the convex set $\As$, with the Cholesky decomposition $\Sigma = L L^T$. Samples $a^s \sim \polsdot$ can be obtained using $a^s = f_\theta(\tilde{\varepsilon}, s) = \mu + L \, \tilde{\varepsilon}$, where the random variable $\tilde{\varepsilon}$ is obtained by sampling from the standard Gaussian distribution $\mathcal{N}_{\Ast}(0, I)$ truncated to the set $\Ast = L^{-1} \big(\As \oplus (- \mu) \big)$. The operator $\oplus$ denotes the Minkowski sum.
\end{proposition}
\begin{proof}
	The result is proven in Appendix~\ref{app:reparam_trick} and illustrated in Fig.~\ref{fig:reparameterization_trick}.
\end{proof}
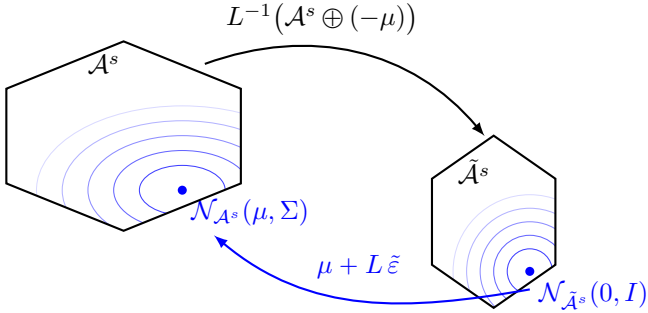
\begin{figure}
	\centering
	\input{figures/reparemerization_trick.tex}
	\caption{The reparameterization trick: We sample the independent random variable $\tilde{\varepsilon}$ from $\mathcal{N}_{\tilde{\As}}(0,1)$ truncated to $\tilde{\As}$,
	and then apply the affine transformation $a^s = \mu + L \, \tilde{\varepsilon}$ to obtain samples from $\polsdot = \mathcal{N}_{\As}(\mu, \Sigma)$, where $\Sigma = L L^T$.}
	\label{fig:reparameterization_trick}
\end{figure}

\subsubsection{Sampling from non-convex sets}
When $\As$ is the union of disjoint convex sets, we draw samples from $\polsdot$ by first selecting a set ${\As}^{(i)}$ with a probability proportional to its relative probability mass $w_i$ (see Prop.~\ref{prop:truncated_policy_dist_union}), then sampling $a^s$ from the chosen convex set with a method from Sec.~\ref{sec:sampling}.
For general non-convex $\As$, geometric random walks are not applicable, leaving rejection sampling as the primary option. In this case, it is advisable to still cap the number of rejection attempts at $M$, and, if possible, resort to a fallback action guaranteed to lie in $\As$, although this introduces a bias.

%% file: figures/reparemerization_trick.tex
\newcommand{\makehexagoncoordsxy}[5]{%
	\pgfmathsetmacro{\xa}{#1 + #3*#4*cos(90)}
	\pgfmathsetmacro{\ya}{#2 + #3*#5*sin(90)}
	\pgfmathsetmacro{\xb}{#1 + #3*#4*cos(150)}
	\pgfmathsetmacro{\yb}{#2 + #3*#5*sin(150)}
	\pgfmathsetmacro{\xc}{#1 + #3*#4*cos(210)}
	\pgfmathsetmacro{\yc}{#2 + #3*#5*sin(210)}
	\pgfmathsetmacro{\xd}{#1 + #3*#4*cos(270)}
	\pgfmathsetmacro{\yd}{#2 + #3*#5*sin(270)}
	\pgfmathsetmacro{\xe}{#1 + #3*#4*cos(330)}
	\pgfmathsetmacro{\ye}{#2 + #3*#5*sin(330)}
	\pgfmathsetmacro{\xf}{#1 + #3*#4*cos(30)}
	\pgfmathsetmacro{\yf}{#2 + #3*#5*sin(30)}
	\xdef\hexagonpath{(\xa,\ya) -- (\xb,\yb) -- (\xc,\yc) -- (\xd,\yd) -- (\xe,\ye) -- (\xf,\yf) -- cycle}
}

\newcommand{\normalellipses}[5]{%
    \fill[#5] (#1,#2) circle [radius=0.05];
    
	\foreach \i in {1,...,6} {
			\pgfmathsetmacro{\factor}{1 - 0.15*(\i-1)}
			\pgfmathsetmacro{\opa}{0.15*(\i-1)}
			\draw[#5, opacity=\opa]
			(#1,#2) ellipse [x radius=#3*\factor, y radius=#4*\factor];
		}
}

\begin{tikzpicture}[scale=1.2]
    \pgfmathsetmacro{\sigmax}{1.9}
    \pgfmathsetmacro{\sigmay}{-1.1}
    \pgfmathsetmacro{\mux}{-3.85}
    \pgfmathsetmacro{\muy}{0.9}

    \pgfmathsetmacro{\posx}{-4.5}
    \pgfmathsetmacro{\posy}{1.5}
    \pgfmathsetmacro{\scalex}{1}
    \pgfmathsetmacro{\scaley}{0.7}

	\begin{scope}
        \makehexagoncoordsxy{\posx}{\posy}{1.5}{\scalex}{\scaley};
        \begin{scope}
            \clip \hexagonpath;
            \normalellipses{\mux}{\muy}{1.9}{1.1}{blue}
        \end{scope}
        \node[color=blue] at (\mux, \muy) [below right] {$\mathcal{N}_{\As}( \mu, \Sigma )$};
        \node at (\posx - 0.5, \posy+1) [below right] {$\As$};

        \draw[thick] \hexagonpath;

        \coordinate (origin) at (0,0);
        \coordinate (B) at ($ (origin.north west) + (0.01,0.1) $);
        \coordinate (A) at (0.5,0.5);
    
        \pgfmathsetmacro{\newx}{\posx - \mux + 0.25}
        \pgfmathsetmacro{\newy}{\posy - \muy - 0.05}
        \pgfmathsetmacro{\newscalex}{\scalex / \sigmax}
        \pgfmathsetmacro{\newscaley}{\scaley / \sigmay}

        \makehexagoncoordsxy{\newx}{\newy}{1.5}{\newscalex}{\newscaley};

        \begin{scope}
            \clip \hexagonpath;
            \normalellipses{0}{0}{1}{1}{blue}
        \end{scope}
        \node[color=blue] at (0, 0) [below right] {$\mathcal{N}_{\Ast}(0, I)$};
        \node at (\newx - 0.5, \newy + 0.8) [below right] {$\tilde{\As}$};

        \draw[thick] \hexagonpath;

        \draw[thick, -{latex}, blue] (0,-0.2) to[out=190, in=-40] (-3.5, 0.4);
        \node[blue] at (-1.9, 0.05) {$\mu + L \, \tilde{\varepsilon}$};

        \draw[thick, -{latex}] (-3.6,2.3) to[out=20, in=130] (-0.5, 1.5);
        \node[] at (-2.3, 2.8) {$L^{-1} \big(\As \oplus (-\mu) \big)$};
	\end{scope}

\end{tikzpicture}

%% file: sections/experiments.tex
\subsection{Numerical Experiments}
We generate a dataset of $6000$ factorized Gaussian distributions truncated to polytopes, with $1000$ random instances for each dimension $d = 2, \ldots, 6$. 
The samples are generated to achieve a balanced distribution of probability mass inside $\As$.
Each polytope $\mathcal{P}$ is formed by intersecting the unit box $[-1, 1]^d$ with $n_P \sim \mathcal{U}(d, 4d)$ random halfspaces. Each halfspace uses a random unit vector $a_j \in \mathbb{R}^d$ and offset $b_j = a_j^\top x_0 + \delta_j$, where $x_0 \sim \mathcal{U}(-0.8, 0.8)^d$ and $\delta_j \sim \mathcal{U}(0.1, 1.0)$. For the Gaussian distribution, we set $\mu = x_1 + c$, with $c$ as the center of $\mathcal{P}$, and obtain $x_1 \sim \mathcal{U}(0, 0.5)^d$, and $\sigma \sim \mathcal{U}(0.1, 1)^d$.

\subsubsection{Integral approximation}\label{sec:exp_numerical_approximation}
We compare the accuracy of our interval-based approximations for the integral $\int_{\tilde{a} \in \As(s)} \pol{\tilde{a}} \mathrm{d}\tilde{a}$ as defined in Sec.~\ref{sec:approximations}. The ground truth is estimated with a precision of $1e-5$ using cubature methods \cite{genz_adaptive_2003}\footnote{We use the implementation provided in the package PySimplicialCubature (https://github.com/stla/PySimplicialCubature).}.
Fig.~\ref{fig:integral_comparison} shows the absolute errors for $d=2, \ldots, 6$, which are normalized to the average integral size in each dimension. The outer approximation shows the highest error for $d > 2$, which is expected, since volume in higher dimensions is concentrated near the boundary of geometries \cite{matousek_volumes_2002}. The combined method exhibits the lowest median and quartile range errors in every dimension, as expected from the derived bounds for the integral in Sec.~\ref{sec:approximations}.
\begin{figure}[h!]
	\centering
	\includegraphics[width=0.95\columnwidth]{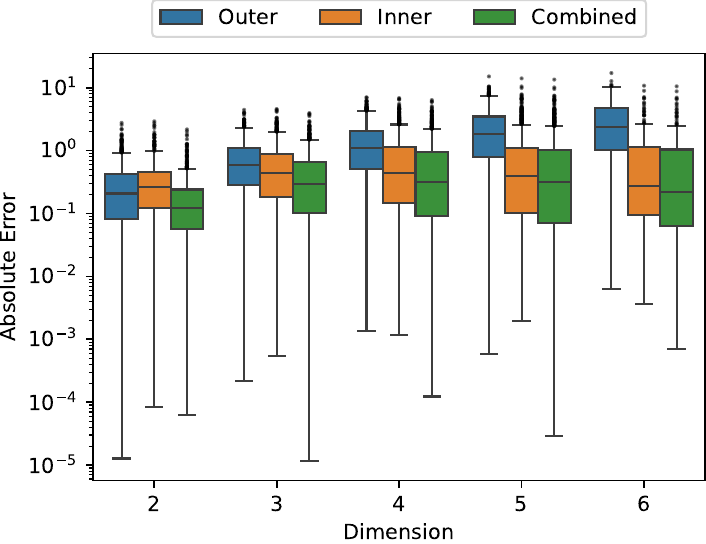}
	\caption{Errors for the integral approximation using the outer, inner, and combined approximation of the polytope.}
	\label{fig:integral_comparison}
\end{figure}

\subsubsection{Sampling}
We assess sampling efficiency (see Sec.~\ref{sec:sampling}) by averaging the time to draw $10$ samples per truncated distribution in the dataset. Fig.~\ref{fig:sampling_time_comparison} presents the sampling times for \ac{RDHR} (geometric random walk), rejection sampling, and our combined sampling strategy with rejection limit $M=100$.
As expected, \ac{RDHR} shows significantly higher median sampling time and lower variance compared to the other methods, and pure rejection sampling exhibits large outliers towards high sampling times. 
The combined method effectively removes these outliers, as intended, and reduces the median sampling time in lower dimensions.
\begin{figure}[h!]
	\centering
	\includegraphics[width=0.95\columnwidth]{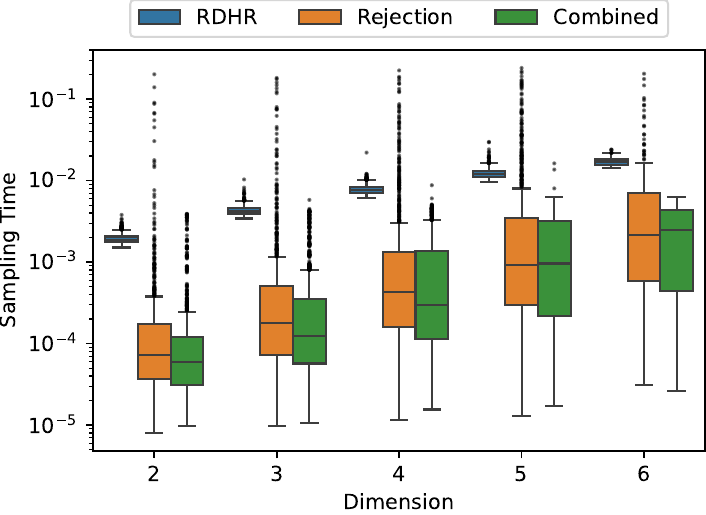}
	\caption{Comparison of the sampling times using the \ac{RDHR}, rejection sampling and the combined sampling method with rejection limit $M=100$.}
	\label{fig:sampling_time_comparison}
\end{figure}

\subsection{Benchmarks}
We evaluate our methods using environments with action constraints that guarantee safety. Computing the safe action sets $\As(s)$ for these environments requires concepts from reachability analysis that are beyond the scope of this paper; the details are provided in Appendix~\ref{app:feasible_action_set}. All environments are implemented in Gymnasium \cite{towers2024gymnasium}, and we use the implementations of \ac{SAC} and \ac{PPO} from Tianshou \cite{tianshou}.

\subsubsection{Seeker}
Our first two environments are based on the Seeker Reach-Avoid environment from \cite[Sec. 5.1.1, A.3.3]{stolz_excluding_2024}. This environment represents a prototypical 2D reach-avoid task in which an agent must reach a goal area while avoiding a single obstacle. We generalize this environment to 2D (\textit{Seeker-2D}) and 3D (\textit{Seeker-3D}) with multiple obstacles and compute feasible action sets as polytopes and intervals instead of zonotopes (see Appendix~\ref{app:seeker_action_set}). We use the slightly modified reward function:
\begin{equation}
	r(a,s) =
	\begin{cases}
		100,                               & \text{if the goal is reached} \\
		-100,                              & \text{if a collision occurs}  \\
		l_\text{prev} - l_\text{curr} - 1, & \text{otherwise},
	\end{cases}
\end{equation}
where $l = \left\| p_{\text{agent}} - p_{\text{goal}}\right\|_2$, and $l_\text{prev}$ and $l_\text{curr}$ are the distances to the goal in the previous and current time step, respectively.

\subsubsection{Quadrotor}

Our third environment (\textit{Quadrotor}) is a 2D quadrotor stabilization task from \cite[Sec. 5.1.2, A.3]{stolz_excluding_2024} in which the agent must stabilize a quadrotor at a goal state $s^*$ while compensating noise. The feasible action sets are computed based on the system dynamics and a desired state set, as detailed in Appendix~\ref{app:quadrotor_action_set}. We use the slightly modified reward function
\begin{equation}
	r(a,s) = \exp\left(-\|s-s^{*}\|_2 - 0.005 \sum_{i=1}^2 a_{i, \text{c}}\right) - 1,
\end{equation}
where $a_{i, \text{c}} = \frac{a_i - a_\mathrm{i,min}}{a_\mathrm{i,range}}$ is the normalized action cost, and $a_\mathrm{i,min}$ is the minimal action, and $a_\mathrm{i,range}$ is the absolute range of actions in dimension $i$.

\subsubsection{Benchmark results}

We optimize hyperparameters for the base \ac{RL} algorithms \ac{SAC} and \ac{PPO}, then use these for training (10 runs per algorithm) with truncated distributions, considering both interval and polytope representations of $\As$. 
Our hypothesis is that accurately estimating $\log \polsdot$ and $\mathcal{H}\big( \polsdot \big)$ improves performance over using the original metrics from $\poldot$, which we test with the following setup:
For polytopes, we compare policies using original metrics (\textit{Og-Poly}) to our outer, inner, and combined approximations (\textit{App-Poly-Out}, \textit{App-Poly-Inn}, \textit{App-Poly-Comb}; see Sec.\ref{sec:approximations}).
For intervals, we compare exact analytic metrics (\textit{Exact-Int}; see Sec.\ref{sec:analytic_solution}) to original values (\textit{Og-Int}). Fig.~\ref{fig:rl_results} shows mean episode returns and $95\%$ confidence intervals for the three environments.

\newcommand\subfigwidth{0.85}

\captionsetup[subfigure]{aboveskip=0pt,belowskip=0pt}
\begin{figure*}[h!]
	\centering
	\begin{subfigure}{\subfigwidth\textwidth}
		\centering
		\includegraphics[width=\linewidth,valign=c]{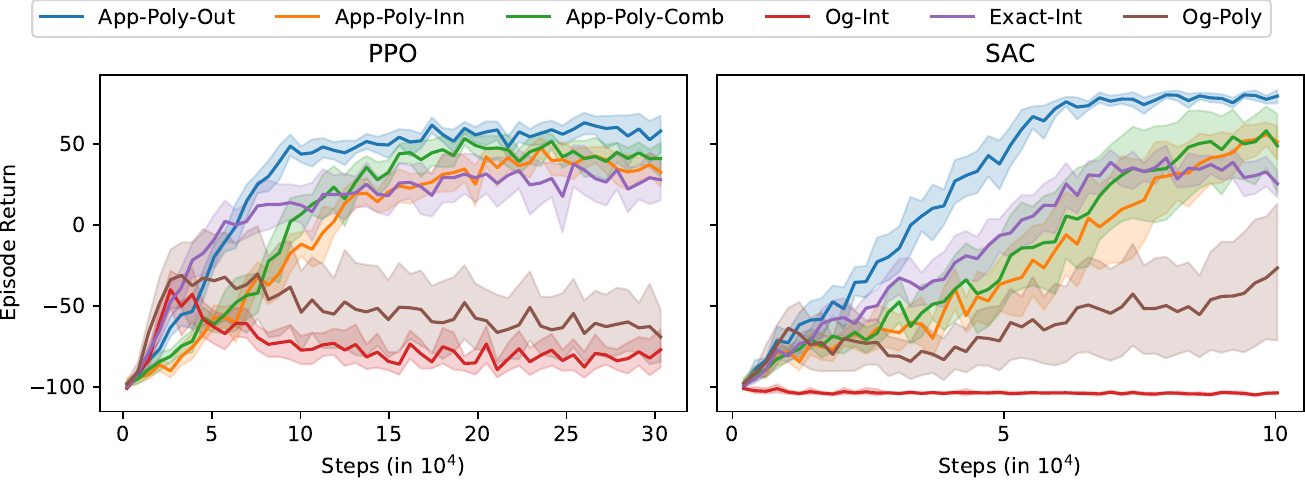}
		\caption{Seeker-2D}
		\label{fig:result1}
	\end{subfigure}
	\begin{subfigure}{\subfigwidth\textwidth}
		\centering
		\includegraphics[width=\linewidth,valign=c]{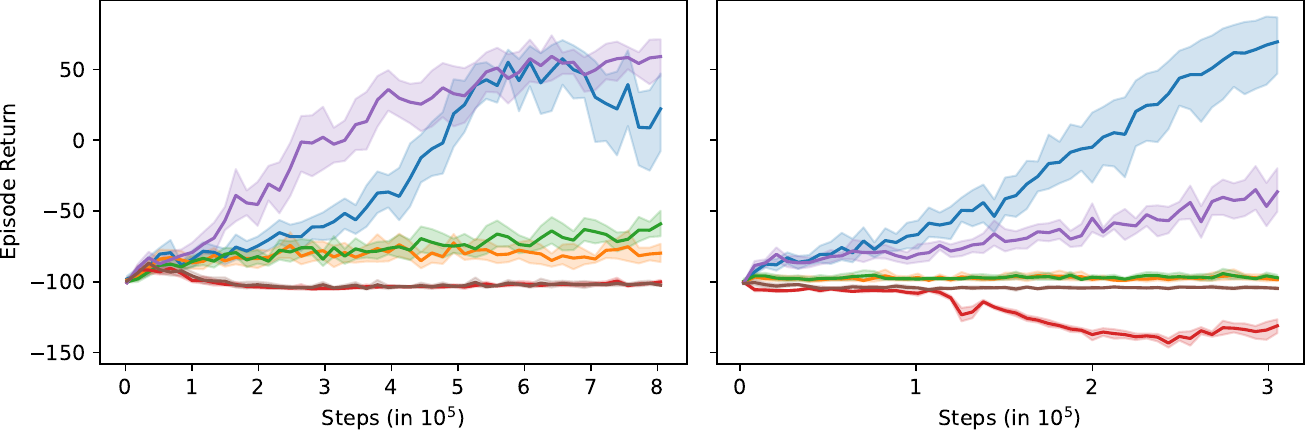}
		\caption{Seeker-3D}
		\label{fig:result2}
	\end{subfigure}
	\begin{subfigure}{\subfigwidth\textwidth}
		\centering
		\includegraphics[width=\linewidth,valign=c]{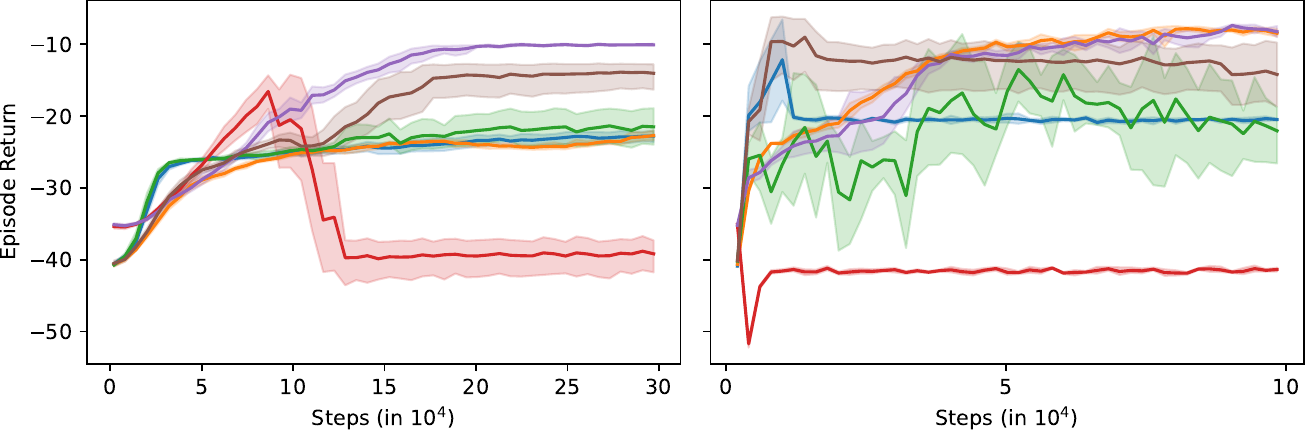}
		\caption{Quadrotor}
		\label{fig:result3}
	\end{subfigure}
	\caption{The mean returns and $95\%$ confidence intervals during \ac{RL} training. \textit{App-Poly-Out}, \textit{App-Poly-Inn}, and \textit{App-Poly-Comb}, refer to the outer, inner, and combined approximations, while \textit{Og-Int} and \textit{Og-Poly} refer to the interval and polytope policies using the original metrics of $\poldot$. \textit{Exact-Int} is the interval policy that computes exact analytic metrics.}
	\label{fig:rl_results}
\end{figure*}

In \textit{Seeker-2D}, both \ac{SAC} and \ac{PPO}, perform best with \textit{App-Poly-Out}, followed by the inner and combined approximations. \textit{Exact-Int} also achieves high returns, while policies using the original values both struggle to learn the task.

The difference between \textit{App-Poly-Out}, and the \textit{App-Poly-Inn} and \textit{App-Poly-Comb} is substantially larger in the \textit{Seeker-3D} environments. The overall returns are also lower than in the 2D case, suggesting significantly higher task complexity. Again, both algorithms fail to learn using the original values, while \ac{SAC} with \textit{App-Poly-Out} achieves the highest returns.

On the \textit{Quadrotor} environment with \ac{PPO}, \textit{Exact-Int} achieves the highest returns, whereas \textit{Og-Int} fails again after initial progress. However, \textit{Og-Poly} achieves results that are slightly better than the approximations.
With \ac{SAC}, \textit{Exact-Int} and \textit{App-Poly-Inn} perform best. \textit{Og-Poly} learns initially, but fails to improve further, and \textit{Og-Int} does not learn at all.

\subsection{Discussion and Limitations}

\subsubsection{Runtime}
Three factors dominate the training time of the algorithms: 1) Computing $\As$ (outside the scope of this work), 2) sampling from $\polsdot$, and 3) estimating the metrics of $\polsdot$, which is dominated by estimating the normalizing constant $Z_{\As}(\theta)$.

The isolated sampling experiment shows that our combined sampling method is faster than pure rejection sampling and \ac{RDHR}, making it the recommended approach when $\As$ is convex. The overall sampling time depends on the set representation of $\As$ and the dimensionality of the action space. The containment check in rejection sampling for polytopes with $m$ halfspaces requires evaluating the halfspace condition $m$ times, and for zonotopes requires solving a linear program \cite{kulmburg_co-np-completeness_2021}.
Geometric random walks compute two boundary points in each step, the complexity of which also depends on the representation of $\As$. Additionally, the sampling time depends on the usual probability mass inside $\As$ during training. When it is low, the algorithm often falls back to geometric random walks, which can result in slow training.

Regarding 3), our estimates of the metrics of $\polsdot$  introduce overhead compared to simply using the non-truncated distribution $\poldot$ as in \cite{stolz_excluding_2024}. The overhead again depends on the representation of $\As$. For intervals, the overhead using the analytic solutions from Sec.~\ref{sec:analytic_solution} is minimal. However, the approximations for general convex sets introduce more overhead. For instance, with polytopes in halfspace representation, the outer approximation with support functions requires solving $2d$ linear programs, while the inner approximation in \eqref{eq:inner_interval_approximation} solves one second-order cone program with polynomial complexity \cite[Sec.~2.5]{boyd2004convex}.

\subsubsection{\ac{RL} performance}
The benchmark results in Fig.~\ref{fig:integral_comparison} clearly indicate that using policy metrics from the original distribution $\poldot$ is ineffective. The strong performance of \textit{Exact-Int} (with analytic solutions) compared to \textit{Og-Int} highlights the benefit of accurate estimates for $\log \polsa$ and $\mathcal{H}\big( \polsdot \big)$.
Two effects explain the bad performance when using $\polsa \approx \pola$ (as in \textit{Og-Int} and \textit{Og-Poly} do). First, the gradient $\nabla_\theta \polsa \approx \nabla_\theta \pola$ is now independent of the normalizing constant $Z_{\As}(\theta)$, hence it might point in the wrong direction. 
Second, through truncation, we often sample actions with very low probability under the original, non-truncated distribution, which can result in a small policy gradient, leading to inefficient learning.

The outer approximation of the polytope with an interval exhibits higher errors compared to the inner or combined approximation in the numerical comparison (see Fig.~\ref{fig:integral_comparison}). Despite this, the \textit{App-Poly-Out} often achieves higher returns than the other two in the \ac{RL} training.
We attribute this to small $\As$ often causing numerical instabilities in the estimation of $Z_{\As}(\theta)$ due to very low probability mass, making the slight positive bias of the outer approximation advantageous.
Future work should further investigate the potential benefit of this bias, and evaluate more precise approximations of $\As$ using the union of intervals (as proposed in Sec.~\ref{sec:non_convex_sets}) numerically, and in \ac{RL} benchmarks.

%% file: sections/conclusion.tex
Achieving effective policy updates, computational efficiency, and a predictable runtime are key challenges in action-constrained \ac{RL}. Previous work proposes truncated distributions, but uses inaccurate metrics and inefficient sampling.
We improve this by developing more accurate numerical estimates for the truncated policy metrics as well as an efficient hybrid sampling approach, and derive a differentiable version of the sampling mechanism, which enables the use of truncated distributions with \ac{SAC}.
The experiments demonstrate that our approach leads to a significantly improved performance across three benchmark environments, underlining the importance of estimating accurate policy metrics in action-constrained \ac{RL}.
We believe that accurate estimations of truncated distributions enable applications beyond those commonly considered in the action-constrained \ac{RL} literature. Future work could, for example, investigate its use in physics-informed algorithms, curriculum learning, or other areas using probability distributions in the learning process, such as representation learning.

%% file: sections/appendix.tex
\newcommand{\inti}{\int_{x \in \mathcal{I}^{(i)}}}
\newcommand{\sumi}{\sum_{i=1}^k}

\section{Proof of Proposition 1}\label{app:truncated_union}
Since the $k$ intervals $\mathcal{I}^{(i)}$ are non-overlapping, the indicator function for the union of intervals $\ui$ can be decomposed into the sum of the individual indicator functions:
\begin{equation*}
	\phi(x; \ui) = \sumi \phi(x; \mathcal{I}^{(i)}).
\end{equation*}
We can substitute this into the expression of the distribution truncated to $\ui$ as
\begin{align*}
	f( & x; \ui)  \overset{\eqref{eq:truncated_general}}{=} \frac{f(x) \sum_{i=1}^k \phi(x; \ii)}{Z_\mathcal{U}}      \\
	   & = \sum_{i=1}^k \frac{f(x) \phi(x; \ii)}{Z_\mathcal{U}}
	= \sum_{i=1}^k \frac{Z_{\mathcal{I}^{(i)}} f(x) \phi(x; \ii)}{Z_{\mathcal{I}^{(i)}} Z_\mathcal{U}}                \\
	   & \overset{\eqref{eq:truncated_general}}{=} \sum_{i=1}^k \frac{Z_{\mathcal{I}^{(i)}}}{Z_\mathcal{U}} f(x; \ii)
	= \sumi w_i \finti,
\end{align*}
where $w_i = \frac{Z_{\mathcal{I}^{(i)}}}{Z_\mathcal{U}}$. This concludes the proof.

\section{Proof of Proposition 2}\label{app:entropy_union}

Since the intervals $\mathcal{I}^{(i)}$ are disjoint, we have that
\begin{equation}\label{eq:proof_2_supp}
	\int_{x \in \ui}  \sum_{i=1}^k \finti \mathrm{d}x = \sumi \inti \finti \mathrm{d}x.
\end{equation}
Using the general expression for the entropy \cite[Eq.~1.1]{shangari_partial_2012}, we can write
\begin{align*}
	 & \mathcal{H}  \big( f(x; \ui) \big) =   - \int_{x \in \ui} f(x; \ui) \log f(x; \ui) \mathrm{d}x                                                                      \\
	 & \overset{Prop.~\ref{prop:truncated_policy_dist_union}, \,\eqref{eq:truncated_policy_dist_union}}{=} -\inti \sumi w_i \finti \log \big( w_i \finti \big) \mathrm{d}x \\
	 & \overset{\eqref{eq:proof_2_supp}}{=} - \sumi \inti w_i \finti \log \left( w_i \finti \right) \mathrm{d}x                                                            \\
	 & = - \sumi w_i \inti \finti \big( \log w_i + \log \finti \big) \mathrm{d}x                                                                                           \\
	 & = - \sumi w_i \log w_i \underbrace{\inti \finti \mathrm{d}x}_1                                                                                                      \\
	 & \quad\quad - \sumi \underbrace{\inti \finti \log \finti \mathrm{d}x}_{-\mathcal{H}\big(\finti \big)}                                                                \\
	 & = - \sumi w_i \log w_i + \sumi \mathcal{H}\big(\finti \big),
\end{align*}
which is equivalent to \eqref{eq:entropy_union}, thus concluding the proof.

\section{Proof of Proposition 3}\label{app:reparam_trick}
For the proposition to be true, we need to show for the reparameterization $a^s = f(\tilde{\varepsilon}, s) = \mu + L \, \tilde{\varepsilon}$ that 1) applying it to the samples of the standard Gaussian distribution $\mathcal{N}(0, I)$ results in the Gaussian distribution $\mathcal{N}(\mu, \Sigma)$, with the Cholesky decomposition $\Sigma = L L^T $, and 2), applying it to all elements in $\Ast$ results in $\As$.

For 1), $f(\tilde{\varepsilon}, s)$ is simply an affine transformation of the standard Gaussian distribution $\mathcal{N}(0, I)$, which results in the Gaussian distribution \cite[Thm.~3.3.3]{tong_multivariate_1990}
\begin{equation}
	\mathcal{N}( L 0 + \mu, L I L^T )
	=\mathcal{N}( \mu, \Sigma ) .
\end{equation}
For 2), we can write $\As = \{ \mu + L \, \tilde{\varepsilon} \mid \tilde{\varepsilon} \in \Ast \}$ as
\begin{align*}
	\As                                                           & = L \Ast \oplus \mu \\
	\As \oplus (-\mu)                                             & = L \Ast            \\
	L^{-1} \big(\As \oplus (-\mu) \big) & = \Ast,
\end{align*}
which shows that applying the transformation to all elements in $\Ast$ results in $\As$, thus concluding the proof.


\section{Computation of the Feasible Action Set}\label{app:feasible_action_set}
The computation of the feasible action set using reachability analysis follows a similar principle as detailed in \cite[Sec.~A.3]{stolz_excluding_2024}, with the modification of $\As$ being represented as a polytope instead of a zonotope. We also assume the existence of a robust control invariant set $\mathcal{S}^r$, which guarantees that there is always an action that keeps the system within its boundaries given $s_0 \in \mathcal{S}^r$.
We summarize the concept and highlight the differences to the zonotope-based computation in the following sections. Interested readers are referred to the original source \cite{stolz_excluding_2024}.

\subsection{Quadrotor}\label{app:quadrotor_action_set}
We introduce a zonotope as $\mathcal{Z} = \{ c + G \beta \mid \Vert\beta \Vert_\infty \leq 1 \} = \langle c, G \rangle_\mathcal{Z}$ with support function for direction $l \in \mathbb{R}^d$
as $\rho_{\mathcal{Z}}(\ell) = \ell^\top c + \|\ell^\top G\|_1$ \cite[Lemma~1]{althoff_support_2016}.
The discrete-time linearized system dynamics of the quadrotor from \cite[Eq.~30]{stolz_excluding_2024} are:
\begin{equation}
	s_{k+1} = A s_k + B a_k + w_k',
\end{equation}
where $A \in \mathbb{R}^{m \times m}$ and $B \in \mathbb{R}^{m \times d}$ are the system matrices, and $w_k' \in \mathcal{W} = \langle c^\mathcal{W}, G^\mathcal{W} \rangle_\mathcal{Z} \subset \mathbb{R}^{m}$ is the disturbance at time step $k$.


Starting from states $\mathcal{S}_0$, and using the feasible inputs $\As$, the set of reachable states of the system at time step $\Delta t$ is $\mathcal{R}_{\Delta t}(\mathcal{S}_0, \As)$. The robust control invariant set is a zonotope $\mathcal{S}^r = \langle c^\mathcal{S}, G^\mathcal{S} \rangle_\mathcal{Z} \subseteq \mathbb{R}^m$.
To ensure $\mathcal{R}_{\Delta t}(\mathcal{S}_0, \mathcal{A}_\mathcal{P}) \subseteq \mathcal{S}^r$, we utilize support functions. For a given direction $\ell$, the containment constraint becomes \cite[Eq.~1]{althoff_support_2016}
\begin{equation}\label{eq:support_function_constraint}
	\ell^\top (A s_k + B a_k) \leq \rho_{\mathcal{S}^r}(\ell) - \rho_{\mathcal{W}}(\ell).
\end{equation}
To account for $\mathcal{W}$ and $\mathcal{S}^r$, we apply the constraint for each direction defined by the generators of the zonotopes, i.e., $L = [ G^\mathcal{W}, G^{\mathcal{S}}]^\top$, where $L \in \mathbb{R}^{n_G \times m}$ and $n_G$ is the total number of generators in $\mathcal{W}$ and $\mathcal{S}^r$.
Further, we represent the action space $\mathcal{A}$ with the constraints $P_\mathcal{A} a \leq p_\mathcal{A}$.
From this and \eqref{eq:support_function_constraint}, we can define the feasible action set as
\begin{equation}
	\begin{aligned}
		 & \As_\mathcal{P}(s_k)=                                                                   \\
		 & \left\{ a \in \mathbb{R}^2 : \begin{bmatrix} P_\mathcal{A} \\ L^TB \end{bmatrix} a \leq
		\begin{bmatrix} p_\mathcal{A} \\ \rho_1 - l_1^T A s_k \\
			\vdots        \\
			\rho_{n_G} - l_{n_G}^T A s_k\end{bmatrix} \right\},
	\end{aligned}
\end{equation}
where $\rho_i = \rho_{\mathcal{S}^r}(l_i) - \rho_{\mathcal{W}}(l_i)$ to simplify notation.


\subsection{Seeker}\label{app:seeker_action_set}

We directly construct the feasible action as a polytope in halfspace representation $\As_\mathcal{P}$ from the optimization problem in \cite[Eq.~35]{stolz_excluding_2024}. The following formulation is generalized to be viable in any dimension $d$, and work for multiple obstacles instead of just one.
Due to the simple dynamics of the seeker environment,
$s_{k+1} = s_k + a_k$,
the boundary collision constraints $s_k + a \in [-10, 10]^d$ are
\begin{equation}
	P_b = \begin{bmatrix} I_d \\ -I_d \end{bmatrix}, \quad p_b = \begin{bmatrix} 10\cdot \mathbf{1}_d - s_k \\ 10 \cdot \mathbf{1}_d + s_k \end{bmatrix},
\end{equation}
where $I_d$ is the identity matrix and $\mathbf{1}_d$ is a vector of ones in $d$ dimensions, respectively.
To avoid the $m$ obstacles $\mathcal{O}_i$ with center $o_i$ and radius $r_i$, the constraints for $i = 1, \ldots, m$ are
\begin{equation}
	P_o = \begin{bmatrix} n_1^\top \\ \vdots \\ n_n^\top \end{bmatrix}, \quad p_o = \begin{bmatrix} b_1 - n_1^\top s_k \\ \vdots \\ b_n - n_n^\top s_k \end{bmatrix},
\end{equation}
where $n_i = \frac{o_i - s_k}{\|o_i - s_k\|}$ and $b_i = n_i^\top o_i - r_i$ following the halfspace approximation from \cite[App.~A.3.3]{stolz_excluding_2024}. The feasible action set is then
\begin{equation}
	\As_\mathcal{P}(s_k) = \left\{a \in \mathbb{R}^d : \begin{bmatrix} P_\mathcal{A} \\ P_b \\ P_o \end{bmatrix} a \leq \begin{bmatrix} p_\mathcal{A} \\ p_b \\ p_o \end{bmatrix}\right\}.
\end{equation}

\section{Hyperparameters}
The following ranges were used for Bayesian hyperparameter optimization with \ac{PPO}:
\begin{itemize}
	\item \textbf{batch\_size:} \{64, 128, 256\},
	\item \textbf{n\_neurons:} \{64, 128, 256\},
	\item \textbf{lr:} log-uniform, $[1 \times 10^{-5}, 1 \times 10^{-3}]$,
	\item \textbf{eps\_clip:} uniform, $[0.1, 0.3]$,
	\item \textbf{repeat\_per\_collect:} integer uniform, $[1, 5]$,
	\item \textbf{ent\_coef:} log-uniform, $[5 \times 10^{-4}, 5 \times 10^{-2}]$,
	\item \textbf{vf\_coef:} uniform, $[0.25, 1.0]$,
\end{itemize}
and with \ac{SAC}:
\begin{itemize}
	\item \textbf{batch\_size:} \{64, 128, 256\},
	\item \textbf{n\_neurons:} \{64, 128, 256\},
	\item \textbf{actor\_lr:} log-uniform, $[1 \times 10^{-5}, 1 \times 10^{-3}]$,
	\item \textbf{critic\_lr:} log-uniform, $[1 \times 10^{-5}, 1 \times 10^{-3}]$,
	\item \textbf{update\_per\_step:} uniform, $[0.5, 1.5]$,
	\item \textbf{gamma:} uniform, $[0.95, 0.999]$,
	\item \textbf{tau:} log-uniform, $[1 \times 10^{-3}, 1 \times 10^{-1}]$,
	\item \textbf{alpha:} uniform, $[0.05, 0.5]$.
\end{itemize}

The hyperparameters obtained over 100 optimizations for the three environments are noted in Tab.~\ref{tab:hp_ppo} and Tab.~\ref{tab:hp_sac}.

\begin{table}[h]
	\caption{PPO Hyperparameters.}
	\centering
	\begin{tabular}{l c c c}
		\toprule
		\textbf{Hyperparameter} & \textbf{Quadrotor} & \textbf{Seeker-2d} & \textbf{Seeker-3d} \\
		\midrule
		batch Size              & 64                 & 128                & 64                 \\
		ent\_coef               & 0.0224             & 0.0013             & 0.0045             \\
		eps\_clip               & 0.1031             & 0.2019             & 0.2233             \\
		learning rate           & 8.35e-4            & 4.77e-4            & 7.90e-4            \\
		n\_neurons              & 256                & 64                 & 256                \\
		repeat\_per\_collect    & 5                  & 4                  & 1                  \\
		vf\_coef                & 0.8413             & 0.8958             & 0.2525             \\
		\bottomrule
	\end{tabular}
	\label{tab:hp_ppo}
\end{table}

\begin{table}[h]
	\caption{SAC Hyperparameters}
	\centering
	\begin{tabular}{l c c c}
		\toprule
		\textbf{Hyperparameter} & \textbf{Quadrotor} & \textbf{Seeker-2d} & \textbf{Seeker-3d} \\
		\midrule
		batch bize              & 64                 & 256                & 64                 \\
		actor\_lr               & 9.41e-4            & 6.23e-5            & 4.02e-5            \\
		critic\_lr              & 4.32e-5            & 6.23e-4            & 3.32e-4            \\
		alpha                   & 0.0509             & 0.1499             & 0.3643             \\
		gamma                   & 0.9931             & 0.9911             & 0.9963             \\
		n\_neurons              & 256                & 64                 & 128                \\
		tau                     & 0.0143             & 0.0563             & 0.0525             \\
		update\_per\_step       & 1.05               & 1.28               & 1.37               \\
		\bottomrule
	\end{tabular}
	\label{tab:hp_sac}
\end{table}

%% file: aaai2026.bib
@inproceedings{stolz_excluding_2024,
  author    = {Stolz, Roland and Krasowski, Hanna and Thumm, Jakob and Eichelbeck, Michael and Gassert, Philipp and Althoff, Matthias},
  booktitle = {Advances in Neural Information Processing Systems},
  doi       = {10.52202/079017-3013},
  pages     = {95067--95094},
  title     = {Excluding the Irrelevant: Focusing Reinforcement Learning through Continuous Action Masking},
  volume    = {37},
  year      = {2024}
}

@book{johnson_continuous_1994,
  edition   = {2nd},
  title     = {Continuous {Univariate} {Distributions}, {Volume} 1},
  author    = {Johnson, N. L. and Kotz, S. and Balakrishnan, N.},
  year      = {1994},
  publisher = {John Wiley \& Sons, Ltd}
}

@article{genz_adaptive_2003,
  title   = {An adaptive numerical cubature algorithm for simplices},
  volume  = {29},
  doi     = {10.1145/838250.838254},
  number  = {3},
  journal = {ACM Transactions on Mathematical Software},
  author  = {Genz, Alan and Cools, Ronald},
  year    = {2003},
  pages   = {297--308}
}

@incollection{robert_monte_2004,
  title     = {Monte {Carlo} {Integration}},
  booktitle = {Monte {Carlo} {Statistical} {Methods}},
  publisher = {Springer},
  author    = {Robert, Christian P. and Casella, George},
  year      = {2004},
  doi       = {10.1007/978-1-4757-4145-2_3},
  pages     = {79--122}
}

@inproceedings{hung_efficient_2025,
  title     = {Efficient {Action}-{Constrained} {Reinforcement} {Learning} via {Acceptance}-{Rejection} {Method} and {Augmented} {MDPs}},
  booktitle = {The {Thirteenth} {International} {Conference} on {Learning} {Representations}},
  author    = {Hung, Wei and Sun, Shao-Hua and Hsieh, Ping-Chun},
  year      = {2025}
}

@inproceedings{brahmanage_flowpg_2023,
  author    = {Brahmanage, Janaka and Ling, Jiajing and Kumar, Akshat},
  booktitle = {Advances in Neural Information Processing Systems},
  pages     = {20118--20132},
  title     = {FlowPG: Action-constrained Policy Gradient with Normalizing Flows},
  volume    = {36},
  year      = {2023}
}

@article{lin_escaping_2021,
  title   = {Escaping from {Zero} {Gradient}: {Revisiting} {Action}-{Constrained} {Reinforcement} {Learning} via {Frank}-{Wolfe} {Policy} {Optimization}},
  volume  = {161},
  journal = {Proceedings of Machine Learning Research},
  author  = {Lin, Jyun Li and Hung, Wei and Yang, Shang Hsuan and Hsieh, Ping Chun and Liu, Xi},
  year    = {2021},
  pages   = {397--407}
}

@article{kasaura_benchmarking_2023,
  title    = {Benchmarking {Actor}-{Critic} {Deep} {Reinforcement} {Learning} {Algorithms} for {Robotics} {Control} {With} {Action} {Constraints}},
  volume   = {8},
  doi      = {10.1109/LRA.2023.3284378},
  number   = {8},
  journal  = {IEEE Robotics and Automation Letters},
  author   = {Kasaura, Kazumi and Miura, Shuwa and Kozuno, Tadashi and Yonetani, Ryo and Hoshino, Kenta and Hosoe, Yohei},
  year     = {2023},
  pages    = {4449--4456}
}

@article{sanket_solving_2020,
  title   = {Solving Online Threat Screening Games using Constrained Action Space Reinforcement Learning},
  volume  = {34},
  doi     = {10.1609/aaai.v34i02.5599},
  number  = {02},
  journal = {Proceedings of the AAAI Conference on Artificial Intelligence},
  author  = {Sanket, Shah and Sinha, Arunesh and Varakantham, Pradeep and Andrew, Perrault and Tambe, Milind},
  year    = {2020},
  month   = {Apr.},
  pages   = {2226-2235}
}

@inproceedings{brahmanage_leveraging_2025,
  title     = {Leveraging Constraint Violation Signals for Action Constrained Reinforcement Learning},
  author    = {Brahmanage, Janaka Chathuranga and Ling, Jiajing and Kumar, Akshat},
  booktitle = {Proceedings of the AAAI Conference on Artificial Intelligence},
  volume    = {39},
  pages     = {15614--15621},
  year      = {2025}
}

@article{krasowski2023provably,
  title   = {Provably Safe Reinforcement Learning: Conceptual Analysis, Survey, and Benchmarking},
  author  = {Hanna Krasowski and Jakob Thumm and Marlon M{\"u}ller and Lukas Sch{\"a}fer and Xiao Wang and Matthias Althoff},
  journal = {Transactions on Machine Learning Research},
  year    = {2023},
}

@book{sutton_reinforcement_2018,
  title     = {Reinforcement Learning: An Introduction},
  author    = {Sutton, Richard S. and Barto, Andrew G.},
  year      = {2018},
  publisher = {MIT Press},
  edition   = {2nd}
}

@inproceedings{sutton_policy_1999,
  title     = {Policy {Gradient} {Methods} for {Reinforcement} {Learning} with {Function} {Approximation}},
  volume    = {12},
  booktitle = {Advances in {Neural} {Information} {Processing} {Systems}},
  author    = {Sutton, Richard S and McAllester, David and Singh, Satinder and Mansour, Yishay},
  year      = {1999}
}

@inproceedings{haarnoja_soft_2018,
  title     = {{Soft} {Actor}-{Critic}: Off-Policy Maximum Entropy Deep Reinforcement Learning with a Stochastic Actor},
  author    = {Haarnoja, Tuomas and Zhou, Aurick and Abbeel, Pieter and Levine, Sergey},
  booktitle = {Proceedings of the 35th International Conference on Machine Learning},
  pages     = {1861--1870},
  year      = {2018}
}

@misc{schulman_proximal_2017,
  title         = {Proximal {Policy} {Optimization} {Algorithms}},
  archiveprefix = {arXiv},
  author        = {Schulman, John and Wolski, Filip and Dhariwal, Prafulla and Radford, Alec and Klimov, Oleg},
  year          = {2017},
  eprint        = {1707.06347}
}

@incollection{gass_hit-and-run_2013,
  address   = {Boston, MA},
  publisher = {Springer},
  title     = {Hit-and-{Run} {Methods}},
  booktitle = {Encyclopedia of {Operations} {Research} and {Management} {Science}},
  author    = {Zabinsky, Zelda B. and Smith, Robert L.},
  year      = {2013},
  doi       = {10.1007/978-1-4419-1153-7_1145},
  pages     = {721--729}
}

@article{chalkis_volesti_2020,
  author  = {Chalkis, Apostolos and Fisikopoulos, Vissarion},
  title   = {volesti: Volume Approximation and Sampling for Convex Polytopes in R},
  journal = {The R Journal},
  year    = {2021},
  doi     = {10.32614/RJ-2021-077},
  volume  = {13},
  issue   = {2},
  issn    = {2073-4859},
  pages   = {642-660}
}

@article{lovasz_hit-and-run_2006,
  title   = {Hit-and-{Run} from a {Corner}},
  volume  = {35},
  doi     = {10.1137/S009753970544727X},
  number  = {4},
  journal = {SIAM Journal on Computing},
  author  = {Lovász, László and Vempala, Santosh},
  year    = {2006},
  pages   = {985--1005}
}

@inproceedings{kingma_auto-encoding_2014,
  author    = {Diederik P. Kingma and
               Max Welling},
  editor    = {Yoshua Bengio and
               Yann LeCun},
  title     = {Auto-Encoding Variational Bayes},
  booktitle = {The Second International Conference on Learning Representations},
  year      = {2014}
}

@inbook{tong_multivariate_1990,
  author    = {Tong, Y. L.},
  title     = {Fundamental Properties and Sampling Distributions of the Multivariate Normal Distribution},
  booktitle = {The Multivariate Normal Distribution},
  year      = {1990},
  publisher = {Springer},
  pages     = {23--61},
  doi       = {10.1007/978-1-4613-9655-0_3},
}

@inproceedings{brimkov_object_2000,
  author    = {Brimkov, Valentin E.
               and Andres, Eric
               and Barneva, Reneta P.},
  title     = {Object Discretization in Higher Dimensions},
  booktitle = {Discrete Geometry for Computer Imagery},
  year      = {2000},
  publisher = {Springer},
  pages     = {210--221}
}

@book{munkres_analysis_1991,
  title     = {Analysis on Manifolds},
  author    = {Munkres, James R.},
  year      = {1991},
  publisher = {Addison--Wesley Publishing Company}
}

@inproceedings{althoff_support_2016,
  author    = {Althoff, Matthias and Frehse, Goran},
  booktitle = {IEEE Conference on Decision and Control},
  title     = {Combining zonotopes and support functions for efficient reachability analysis of linear systems},
  year      = {2016},
  pages     = {7439--7446}
}

@article{tianshou,
  author  = {Jiayi Weng and Huayu Chen and Dong Yan and Kaichao You and Alexis Duburcq and Minghao Zhang and Yi Su and Hang Su and Jun Zhu},
  title   = {Tianshou: A Highly Modularized Deep Reinforcement Learning Library},
  journal = {Journal of Machine Learning Research},
  year    = {2022},
  volume  = {23},
  number  = {267},
  pages   = {1--6},
}

@article{towers2024gymnasium,
  title   = {Gymnasium: A Standard Interface for Reinforcement Learning Environments},
  author  = {Towers, Mark and Kwiatkowski, Ariel and Terry, Jordan and Balis, John U and De Cola, Gianluca and Deleu, Tristan and Goul{\~a}o, Manuel and Kallinteris, Andreas and Krimmel, Markus and KG, Arjun and others},
  journal = {arXiv preprint arXiv:2407.17032},
  year    = {2024}
}

@article{kulmburg_co-np-completeness_2021,
  title    = {On the co-{NP}-completeness of the zonotope containment problem},
  volume   = {62},
  issn     = {0947-3580},
  journal  = {European Journal of Control},
  author   = {Kulmburg, Adrian and Althoff, Matthias},
  year     = {2021},
  pages    = {84--91},
}

@article{theile_learning_2024,
  title    = {Learning to {Generate} {All} {Feasible} {Actions}},
  volume   = {12},
  doi      = {10.1109/ACCESS.2024.3376739},
  journal  = {IEEE Access},
  author   = {Theile, Mirco and Bernardini, Daniele and Trumpp, Raphael and Piazza, Cristina and Caccamo, Marco and Sangiovanni-Vincentelli, Alberto L.},
  year     = {2024},
  pages    = {40668--40681}
}

@book{michalowicz_handbook_2013,
  title     = {Handbook of {Differential} {Entropy}},
  publisher = {Chapman \& Hall/CRC},
  author    = {Michalowicz, Joseph Victor and Nichols, Jonathan M. and Bucholtz, Frank},
  year      = {2013}
}

@incollection{cover_differential_2005,
  title     = {Differential {Entropy}},
  booktitle = {Elements of {Information} {Theory}},
  publisher = {John Wiley \& Sons, Ltd},
  author    = {Cover, Thomas M. and Thomas, Joy A.},
  year      = {2005},
  pages     = {243--259}
}

@article{shangari_partial_2012,
  title    = {Partial monotonicity of entropy measures},
  volume   = {82},
  issn     = {0167-7152},
  doi      = {https://doi.org/10.1016/j.spl.2012.06.029},
  number   = {11},
  journal  = {Statistics \& Probability Letters},
  author   = {Shangari, Dhruv and Chen, Jiahua},
  year     = {2012},
  pages    = {1935--1940}
}

@incollection{matousek_volumes_2002,
  title     = {Volumes in {High} {Dimension}},
  booktitle = {Lectures on {Discrete} {Geometry}},
  publisher = {Springer},
  author    = {Matoušek, Jiří},
  year      = {2002},
  doi       = {10.1007/978-1-4613-0039-7_13},
  pages     = {311--328}
}

@article{murty_np-complete_1987,
  title    = {Some {NP}-complete problems in quadratic and nonlinear programming},
  volume   = {39},
  issn     = {1436-4646},
  doi      = {10.1007/BF02592948},
  number   = {2},
  journal  = {Mathematical Programming},
  author   = {Murty, Katta G. and Kabadi, Santosh N.},
  month    = jun,
  year     = {1987},
  pages    = {117--129}
}

@phdthesis{althoff_reachability_2010,
  author   = {Althoff, Matthias},
  title    = {Reachability Analysis and its Application to the Safety Assessment of Autonomous Cars},
  year     = {2010},
  school   = {Technische Universität München},
  pages    = {221},
  language = {en},
  note     = {},
  doi      = {}
}

@book{boyd2004convex,
  title={Convex optimization},
  author={Boyd, Stephen and Vandenberghe, Lieven},
  year={2004},
  publisher={Cambridge university press}
}
